\documentclass[letterpaper, 10 pt, conference]{ieeeconf}  

\IEEEoverridecommandlockouts                              

\usepackage{bm}
\usepackage{amsmath}
\usepackage{amssymb}

\usepackage{amsthm}
\usepackage{algpseudocode}
\usepackage{algorithm}
\usepackage{bm}
\usepackage{booktabs}
\usepackage{cite}
\usepackage{comment}
\usepackage{graphics}
\usepackage{mathtools}
\usepackage{float}
\usepackage{graphicx}
\usepackage{url}
\usepackage{hyperref}
\hypersetup{colorlinks, linkcolor=blue}
\usepackage[font=small]{caption}
\usepackage{xcolor}
\usepackage{lipsum}

\usepackage{tikz}
\usetikzlibrary{positioning}

\title{\LARGE \bf
Funnel planning and control for manipulators using neural Signed Configuration Distance Functions
$^{*}$
}

\author{Bernhard~Wullt$^{1}$, Christos~K.~Verginis$^{2}$, Per~Mattsson$^{3}$, Mikael~Norrlöf$^{1}$, Thomas~B.~Schön$^{3}$
\thanks{$^{*}$This research was supported by the \emph{Wallenberg AI, Autonomous Systems and Software Program (WASP)} funded by Knut and Alice Wallenberg Foundation.}
\thanks{$^{1}$Bernhard Wullt and Mikael Norrlöf are with ABB Robotics, Sweden {\tt\small name.surname@se.abb.com}}%
\thanks{$^{2}$Christos K. Verginis is with the Division of Signals and Systems, Department of Electrical Engineering, Uppsala University, Uppsala, Sweden. {\tt\small christos.verginis@angstrom.uu.se}}%
\thanks{$^{3}$Per Mattsson and Thomas B. Schön are with the Division of Systems and Control, Department of Information Technology, Uppsala University, Uppsala, Sweden. \tt\small name.surname@it.uu.se} %
}

\begin{document}

\def\nrDOF{n}
\def\nrDOFmain{n_m}
\def\nrDOFwrist{n_w}
\def\timeFinal{t_f}
\def\funnelGain{k_1}
\def\funnelVelocityGain{k_2}
\def\c{\bm{c}}
\def\p{\bm{p}}
\def\q{\bm{q}}
\def\qdot{\dot{\bm{\q}}}
\def\x{\bm{x}}
\def\xdot{\dot{\x}}

\def\vecGrav{\bm{g}}
\def\vecGravEst{\hat{\vecGrav}}

\def\xStart{\x_{\text{s}}}
\def\xGoal{\x_{\text{g}}}

\def\qddot{\ddot{\q}}

\def\dist{\bm{w}}
\def\u{\bm{u}}

\def\matR{\bm{R}}
\def\matT{\bm{T}}
\def\matI{\bm{I}}

\def\matMass{\bm{M}}
\def\matMassEst{\hat{\matMass}}

\def\matCor{\bm{C}}
\def\matCorEst{\hat{\matCor}}
\def\spaceR{\mathbb{R}}
\def\setB{\mathcal{B}}
\def\setC{\mathcal{C}}
\def\setCf{\setC_{\text{f}}}
\def\setCo{\setC_{\text{o}}}

\def\setD{\mathcal{D}}
\def\setFK{\mathcal{FK}}
\def\setO{\mathcal{O}}
\def\setN{\mathcal{N}}
\def\setP{\mathcal{P}}
\def\setU{\mathcal{U}}

\def\setG{\mathcal{G}}
\def\setT{\mathcal{T}}
\def\setW{\mathcal{W}}
\def\setMA{\mathcal{MA}}

\def\setX{\mathcal{X}}

\def\funcD{\phi}
\def\funcSD{\psi}
\def\funcSDmain{\funcSD^m}
\def\funcSDwristLocal{\funcSD^{lw}}
\def\funcSDwrist{\funcSD^{w}}
\def\funcNSD{\funcSD_{\bm{\theta}}}
\def\funcNSDmain{\funcNSD^m}
\def\funcNSDwristLocal{\funcNSD^{\parallel}}

\def\funcNSDwrist{\funcNSD^w}

\def\funcFunnel{\rho}
\def\funcErr{\bm{e}}

\def\funcTraj{\bm{\tau}}

\def\funcCard{\text{card}}
\def\funcHyp{\text{hypot}}
\def\funcDiag{\text{diag}}

\def\intRange{\mathbb{I}}
\def\zeroVec{\bm{0}}

\newcommand{\norm}[1]{\lVert#1\rVert}

\newtheorem{theorem}{Theorem}
\maketitle
\thispagestyle{empty}
\pagestyle{empty}
\begin{abstract}
Robust motion planning and control for manipulators is the problem of guaranteeing safe execution and collision-free motion under disturbances.
The Kinodynamic Motion Planning via Funnel Control (KDF) is a recently proposed planner that comes with robustness guarantees and is particularly attractive for collision-free control as it guarantees closed-loop evolution inside a funnel, i.e., a collision-free region around a nominal trajectory. Despite its attractive theoretical properties, the current approach for computing funnels is computationally demanding, which limits its expressiveness and performance.
We address its current limitations by proposing a novel funnel planning and control solution. Our key tool to enable this is a neural Signed Configuration Distance Function (nSCDF) which allows us to estimate collision-free balls and gradients in a computationally efficient way. We use the nSCDF in our planner to compute a nominal trajectory, where path length and clearance are balanced. To track the trajectory with collision-avoidance guarantees, we design a funnel controller that exploits the geometric information from the nSCDF.
From numerical experiments, we observe that our planner produces nominal paths $2$ times faster, while at the same time obtaining $20 \%$ shorter path length and $5$ times the clearance. Finally, we also observe improved motion performance and robustness properties when using the nSCDF for funnel control.
\end{abstract}
\section{Introduction}
Robust motion planning and control for robot manipulators is the process of designing an algorithm that guides the motion of the arm from a given start to a goal position while guaranteeing safe motion under disturbances.
It is a fundamental problem in robotics, part of any application which requires reliable and safe autonomous operations, e.g., bin picking, space missions.
The problem is challenging, due to the many sources of disturbances that arise when controlling the manipulator, e.g., model-errors and external disturbances. Furthermore, collision-free motion is a major challenge, where a limiting factor is the lack of computationally efficient obstacle representations in the configuration space \cite{configuration_space}. To cope with these difficulties, the standard approach \cite{lavallebook} is to decouple the problem into subproblems. First, a collision-free path \cite{rrt, prm} is produced by making heavy use of collision detection, which is time-scaled \cite{topps_cvx, toppra} and then tracked with a feedback controller~\cite{rob_mod_plan_cont}. However, when the nominal trajectory is tracked, the robot is easily pushed off track due to the disturbances. Since only the nominal path is verified to be collision-free, once the robot deviates from the path, collision may occur, which is not acceptable for safety critical scenarios. The Kinodynamic Motion Planning via Funnel Control (KDF) framework, developed in \cite{kdf}, was introduced to address the lack of robustness guarantees for collision-free control of manipulators. It guarantees closed-loop evolution in a region around a nominal trajectory, i.e., a funnel. However, since the funnel must be collision-free, the planning of the funnel becomes a major challenge. The current approach finds a funnel by using extended free space sampling (EFSS), where it approximately verifies a region through a lot of sampling and collision checking. EFSS is then integrated into a sampling-based planner to find a funnel. Although the approach is successful, it makes heavy use of collision detection and therefore it is computationally expensive. To lower the computational time, the funnel width is typically reduced and post-processing ignored. This is problematic, since it typically results in longer motion and reduced motion performance.
\begin{figure}
    \centering
    \includegraphics[width=0.3\textwidth]{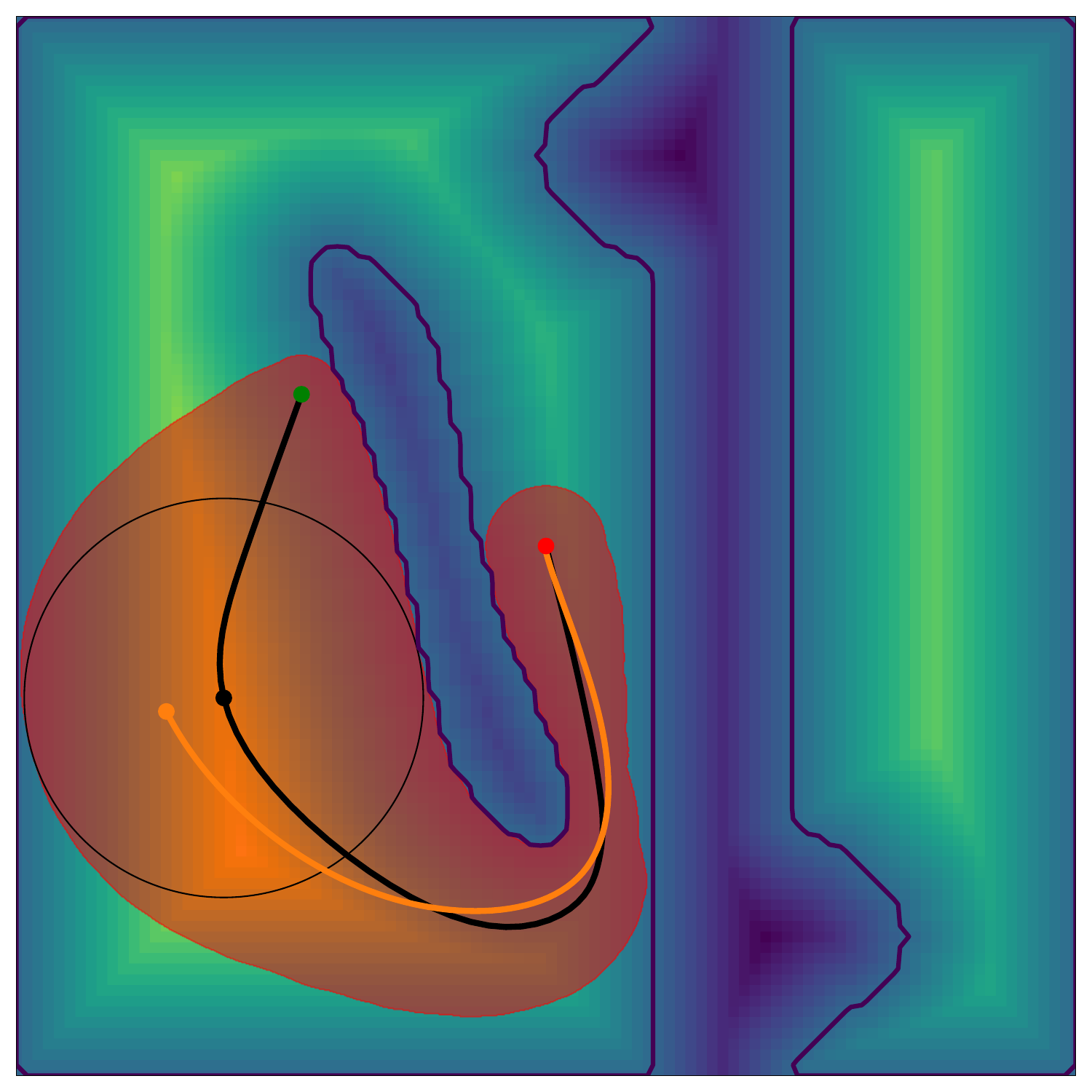}
    \caption{
    Signed distance in the configuration space (brighter colors represent larger distance); the thick purple line (obstacle boundary) is the zero isoline. Our method computes a smooth trajectory (black curve) which balances path length and clearance. To track the trajectory, we develop a funnel controller such that the robot's closed-loop trajectory (orange curve) stays within a funnel defined by the signed distance function (black circle and transparent red region).
    }
    \label{fig:intro}
\end{figure}
\\\\
To overcome the computational bottleneck in KDF, we replace the repeated usage of collision checking with a differentiable distance function, a neural Signed Configuration Distance function (nSCDF) \cite{pbrm}. The nSCDF allows for fast computation of estimates of collision-free regions and gradients. Equipped with this tool, we propose a novel funnel planning and control solution, illustrated in Figure \ref{fig:intro}. Our funnel planner uses the nSCDF to efficiently produce a nominal trajectory, where we balance both path length and clearance. The funnel controller then tracks the trajectory while staying within the funnel as computed by the nSCDF. We summarize our main contributions below.
\begin{itemize}
    \item A novel nSCDF for 6 DOF manipulators that can represent parameterized obstacle geometries.
    \item A funnel planner that uses the nSCDF to produce a nominal trajectory where both path length and clearance is balanced.
    \item A funnel control formulation that exploits the geometric information from the nSCDF to maintain collision-free motion under bounded, unknown disturbances.
\end{itemize}
We open-source our solution for future development:
\begin{center}
{\url{https://github.com/berwul/nSCDF_fc}}.
\end{center}
\section{Related work}
\def\dInPlane{d_{\parallel}}
\def\dNormal{d_{\perp}}
\noindent
\textit{Path planning} for manipulators is done in the configuration space \cite{configuration_space}, i.e., the joint space. The problem is to find a geometrical collision-free path, i.e., ignoring dynamics. The problem is hard, since it typically is high dimensional and lacks expressive collision avoidance representations. The standard approach is to rely on implicit representations through forward kinematics and collision detection~\cite{ericson2004real}. Sampling-based planners~\cite{rob_mod_plan_cont} are a common choice of planners to use for manipulators. The core idea is to grow a graph by iteratively expanding it in collision-free directions, which are collision checked, until the start and goal can be connected. The approach is computationally intensive, yet, effective, being able to scale to higher DOFs. The core planners \cite{rrt, prm} typically result in jerky paths. There are asymptotically optimal sampling-based planners~\cite{rrt_star}, which, however, usually add a considerable amount of computation. A more practical approach is instead to rely on post-processing, e.g., using optimization \cite{schulman2014motion} and shortcutting~\cite{rrt_sc}.
\\
\textit{Kinodynamic motion planning} \cite{kinodynamic} couples the spatial and temporal dimensions in the planning. The sampling-based planning approach allows for kinodynamic motion planning~\cite{rrt_kd} with asymptotically optimal planners \cite{rrt_star_kd, rrt_star_kd_linear}. However, the resulting problem becomes much harder due to, e.g., increased dimensionality and including ODE solvers in the planning. Thus, additionally accounting for disturbances in the planning would make it even more computationally heavy.
\\
\textit{Robustness} is typically not explicitly accounted for in standard motion planning, despite being inherent to the problem. Model predictive control with robustness guarantees have been proposed in \cite{safe_n_fast_rmpc, wullt2026robust}, but a priori knowledge of the disturbances is assumed. 
In contrast to this, KDF~\cite{kdf} makes only minimal assumptions on the disturbances, e.g., boundedness of the disturbances. It decouples motion planning into geometrical planning and control. A key difference compared to standard path planning is that it plans for paths with a user-specified distance margin, i.e., clearance, where the distance margin is expressed in the configuration space. By doing so, a collision-free region, i.e., funnel, can be formed around the resulting path. The motion is then guaranteed to evolve in the region by using funnel control. The strong theoretical guarantees are attractive, but the key problem here is how to efficiently perform the funnel planning.
\\
\textit{Computing collision-free regions} has been actively researched in the last decade \cite{iris_c, IRIS_nlp, cdf, pbrm}. Optimization-based approaches \cite{iris_c, IRIS_nlp} focus on maximizing collision-free volume, but they offer no way to produce a continuous funnel. Learning-based approaches \cite{cdf, pbrm} rely on learning a distance function, from which collision-free balls can be formed. The core data generation process in \cite{cdf} is to project samples on the collision boundary through optimization. However, as the boundary is high dimensional, it requires a lot of samples, making the approach intractable to general obstacle shapes. In contrast, in \cite{pbrm} the data generating approach computes distances from meshes, which does not scale to high-dimensional manipulators.
\section{Notation}
\label{sec:notation}
\def\geom{\bm{\sigma}}
\def\geomWC{\geom_{\text{lw}}}
\def\setGeom{\Sigma}
\def\funcColl{\eta}
\def\setCoBnd{\partial\setCo}
\def\frameWrist{\bm{T}}
\def\setOlw{\setO_{\text{lw}}}
\def\collSign{s}
\def\qSamp{\tilde{\q}}
\def\qSampMain{\tilde{\q}_m}
\def\qSampWrist{\tilde{\q}_w}
\def\dMin{d_{\text{min}}}
\def\NrGridPoints{N}
We denote the set of integers $a$ to $b$, i.e., $\{a, a+1, \hdots, b \}$ by $\intRange_{a:b}$. The sets of non-negative and positive real numbers are denoted by $\spaceR_+$ and $\spaceR_{++}$, respectively. For a vector $\x\in\spaceR^n$, we denote the 2-norm as $\norm{\x}=\sqrt{\x^\top\x}$. We denote the sub-vector, which contains elements $a$ to $b$, as $\x_{a:b}$. The operation of stacking two column vectors is expressed as $(\bm{a}, \bm{b}) = [\bm{a}^\top, \bm{b}^\top]^\top$. The identity matrix is denoted as $\matI$ and the zero vector is denoted as $\zeroVec$. The function $\funcCard(\cdot)$ returns the cardinality of a set. Finally, the function $\funcDiag(\cdot) : \spaceR^n \mapsto \spaceR^{n\times n}$, maps an $n$ dimensional vector into a diagonal matrix.
\section{Problem Formulation}
We limit ourselves to manipulators with six degrees of freedom (DOF). 
Numerous robotic manipulators have 6 DOF, typically divided into a set of main links, comprising 3 DOF, and a wrist consisting of the remaining 3 DOF. The manipulator operates in the world space, $\setW \subset \spaceR^3 $, and the poses of the links are described by a configuration $\q \in \setC$, where $\setC \subset \spaceR^n$ is the configuration space with $\nrDOF=6$. Without loss of generality, we refer to the first three joints, $\q_m \in \spaceR^3$, as \textit{main joints}, with the corresponding links as \textit{main links}, and the remaining joints, $\q_w \in \spaceR^{3}$ as \textit{wrist joints}, with the corresponding links as \textit{wrist links}. We define the wrist frame $\{w\}$ as the interfacing frame between the main and wrist links. The pose of the wrist frame is a function of the main joints, which we express as $\matT(\q_m)$, where $\matT \in \spaceR^{4\times4}$ is a homogeneous transformation matrix.
\\\\
We define the state of the system as $\x=(\q, \qdot) \in \setC \times \spaceR^\nrDOF = \setX$, where $\setX$ is the state space. We access the configuration of the state by the notation $\q(\x)$. The state evolves according to the second-order dynamics
\begin{align}
    \label{eq:dyn_snd}
    \matMass(\q)\qddot + \matCor(\q,\qdot)\qdot + \vecGrav(\q) + \dist(\q,\qdot,t) = \u,
\end{align}
where $\u\in\spaceR^n$ is the robot's control input, $\matMass:\setC \mapsto \spaceR^{n \times n}$ is the mass matrix function, $\matCor:\setC \times \spaceR^n \mapsto \spaceR^{n \times n}$ is the Coriolis matrix function and $\vecGrav:\setC \mapsto \spaceR^n$ is the gravity function. The function $\dist:\setC \times \spaceR^n \times \spaceR_+ \mapsto \spaceR^n$ represents model uncertainties and external disturbances.
We assume that $\matMass$, $\matCor$, and $\vecGrav$ are locally Lipschitz mappings in their arguments, while $\dist$ is locally Lipschitz in $\q$, $\qdot$ and uniformly bounded in $t$. We consider that the functions $\matMass$, $\matCor$, $\vecGrav$ are uncertain, i.e., only estimates $\matMassEst$, $\matCorEst$, $\hat{\vecGrav}$ are available (e.g., via estimates of geometric and dynamic parameters), whereas $\dist$ is entirely unknown. We note that the collision-free guarantees of the proposed algorithm are independent of the system dynamics and valid even without such estimates, i.e., $\matMassEst = \matI$, $\matCorEst = \hat{\vecGrav} = \zeroVec$. 
\\\\
\def\NrObs{K}
The robot is surrounded by $\NrObs$ obstacles $\setO = \bigcup_{i=1}^{\NrObs} \setO_i \subset \setW$. The set of points covered by the robot bodies in the configuration $\q$ is $\setFK(\q) \subset \setW$. The configuration space representation of the $i$-th obstacle is $\setC\setO_i = \{\q \in \setC \: |\: \setFK(\q) \cap \setO_i \neq \emptyset \}$. The obstacle space is formed as $\setCo = \bigcup_{i=1}^{\NrObs} \setC \setO_i$. The complement is referred to as the free space, $\setCf = \setC \setminus \setCo$.
\\\\
\textit{Problem definition: } Given a start state,  $\x(0) = \xStart = (\q_s, \zeroVec) \in \setX$, and a goal state, $\xGoal = (\q_g,\zeroVec) \in \setX$, design a control algorithm, $\u(t)~:~[0, \timeFinal] \mapsto \setU $, where $\timeFinal \in \spaceR_+$ is the execution time, such that the evolution of~\eqref{eq:dyn_snd}, satisfies $\x(t_f)= \x_g$, while being collision-free at all times, $\q(\x(t)) \in \setCf, \: \forall t \in [0, \timeFinal]$.
\section{Signed Configuration Distance Function}
\label{sec:scdf}
In the following, we present a learning-based approach to approximate a high-dimensional signed configuration distance function (SCDF). We start by introducing the SCDF in Section~\ref{sec:scdf:bg}. Having presented its properties, we argue that the derivation of an analytical SCDF is nontrivial. Hence, we employ an approximation via a neural representation, i.e., a neural SCDF (nSCDF), which we present in Section~\ref{sec:scdf:neural_scdf}. With the nSCDF, we address how we can compute the necessary data to form the representation in Section~\ref{sec:scdf:wrist_distances}.
\subsection{Background}
\label{sec:scdf:bg}
A fundamental tool in motion planning is a collision detector, which we define as
\begin{equation}
\funcColl(\q, \: \setO) = 
\begin{cases}
-1 & \text{if } \setFK(\q) \cap \setO \neq \emptyset, \\
\hphantom{-} 1 & \text{otherwise.}
\end{cases}
\end{equation}
It can typically be evaluated fast and reliably, but on the downside, the method does not provide any information regarding the proximity of $\q$ to $\setO$. In contrast to the standard collision detector, we define the SCDF as
\begin{align}
\funcSD(\q, \: \setO) &= 
\begin{cases}
-\funcD(\q, \: \setO) & \text{if } \q \in \setCo, \\
\hphantom{-} \funcD(\q, \: \setO) & \text{otherwise},
\end{cases}
\end{align}
together with
\begin{equation}
  \funcD(\q, \setO) = \min_{\q_b \in \setCoBnd} \: \norm{\q - \q_b}.
\end{equation}
In the above, we denote the boundary of the obstacle region by $\setCoBnd$. In the following, we will occasionally omit the obstacle as input to the SCDF, only adding it where clarity is needed. The SCDF is similar to a standard collision detector as it provides the collision status. However, it also gives us the closest \textit{configuration} distance to the obstacle boundary. We can use this to express a collision-free region. From a configuration $\q_c$, we form a collision-free ball according to
\begin{equation}
    \label{eq:ball}
    \setB(\q) = \{\q_c \in \spaceR^\nrDOF \; | \; \norm{\q_c - \q}  \le \funcSD(\q) \}.
\end{equation}
Furthermore, if the SCDF is a differentiable function, then the gradient allows us to further grow the ball's volume by stepping along its gradient~$\nabla_{\q} \funcSD$. The SCDF is flexible, since it allows us to combine multiple signed distance values by simply taking the smallest distance. How to form an analytical representation of the SCDF is, however, nontrivial. We instead propose to approximate it through machine learning, which we discuss in the next section.
\subsection{Neural representation}
\label{sec:scdf:neural_scdf}
In this section we will show how the SCDF can be approximated using a neural representation, $\funcSD \approx \funcNSD$, where $\funcNSD$ is termed a neural SCDF (nSCDF). Since we use a neural representation, we have to restrict ourselves to geometrical shapes that can be represented by a parameter vector $\geom \in \setGeom$, where $\setGeom$ is the set of geometrical shapes for a particular family of shapes, e.g., spheres, axis-aligned boxes and ellipsoids. With this, the obstacles are formally parameterized as $\setO(\geom)$, but in the following we simply use the shape vector $\geom$ to implicitly represent the obstacle $\setO$. 
\\\\
As presented in the previous section, we can combine multiple signed distances through the min function. We exploit this property and decompose the nSCDF as
\begin{equation}
\label{eq:scdf:min}
\funcNSD(\q, \geom) = \text{min}(\funcNSDmain(\q_m, \geom), \: \funcNSDwrist(\q_m, \q_w, \geom)).
\end{equation}
In the above, the function $\funcNSDmain: \spaceR^3 \times \setGeom \mapsto \spaceR$, is an nSCDF that only considers the collision distance for the main links. We refer to it as the \textit{main} nSCDF. We only require 3~DOF to describe the poses of the main links. Our previous work~\cite{pbrm}, can be used to train a neural representation, since it does not require all the 6~DOF. We now briefly review this method. 
\\\\
The core idea in \cite{pbrm} is the data generation process. For a given shape parameter vector, we compute a mesh that approximates the obstacle boundary by relying on collision detection and a surface reconstruction algorithm, e.g.,~ \cite{marchingcubes}. With the mesh, we can efficiently compute signed distances expressed in the configuration space for arbitrary configurations. To compile a labeled dataset, we simply repeat the previous process for a large number of sampled parameter vectors and configurations, from which we can train a neural network representation. For more details, see~\cite{pbrm}.
\\\\
In \eqref{eq:scdf:min}, the function $\funcNSDwrist: \spaceR^6 \times \setGeom \mapsto \spaceR$ is an nSCDF that only considers the distance of the wrist links and is referred to as a \textit{wrist} nSCDF. To describe the poses of the wrist links, all 6 DOF are required. As such, the previously outlined data generation process \cite{pbrm} will be intractable. The main challenge is how to retrieve signed distances in a computationally efficient way, which we address next.
\subsection{Computing signed distances for the wrist links}
\label{sec:scdf:wrist_distances}
To efficiently compute the signed distances for a high-dimensional obstacle boundary, we propose to use a \textit{slice-based} representation, which we conceptualize in Figure~\ref{fig:method:conceptual} and discuss next.
\begin{figure}[H]
\centering
\begin{tikzpicture}
\node[anchor=south west,inner sep=0] at (0,0) {\includegraphics[height=.15\textheight, trim={0.cm 0.cm 0.cm 0.cm},clip]{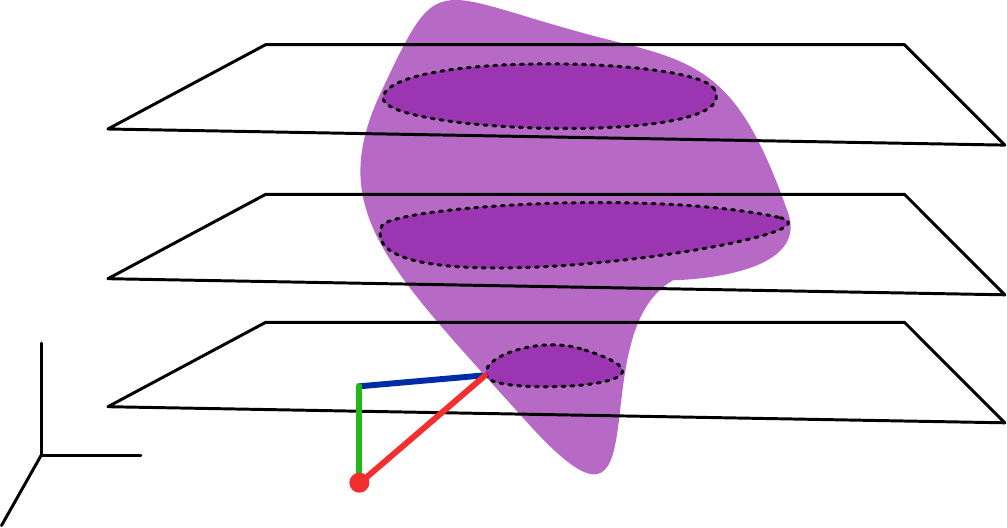}};
\node at (2.2, .6) {$\dNormal$};
\node at (2.5, 1.2) {$\dInPlane$};
\node at (2.9, 0.5) {$d$};
\end{tikzpicture}
\caption{
A conceptual illustration of the representation based on slices. The high-dimensional obstacle region (purple) is discretized into slices (dark purple), by forming the intersection with the hyperplanes (black). From a given query point (red point) we compute the distance to a slice (red line), by combining the normal plane distance (green line) and the in-plane distance to the slice boundary (blue line). The smallest distance from all the slices results in our estimate of the distance to the high-dimensional boundary.
}
\label{fig:method:conceptual}
\end{figure}
\begin{figure*}[t!]
\centering
\begin{tikzpicture}
\node[anchor=south west,inner sep=0] at (0,0) {\includegraphics[height=.35\textwidth, trim={0.25cm 6cm 0.25cm 6cm},clip]{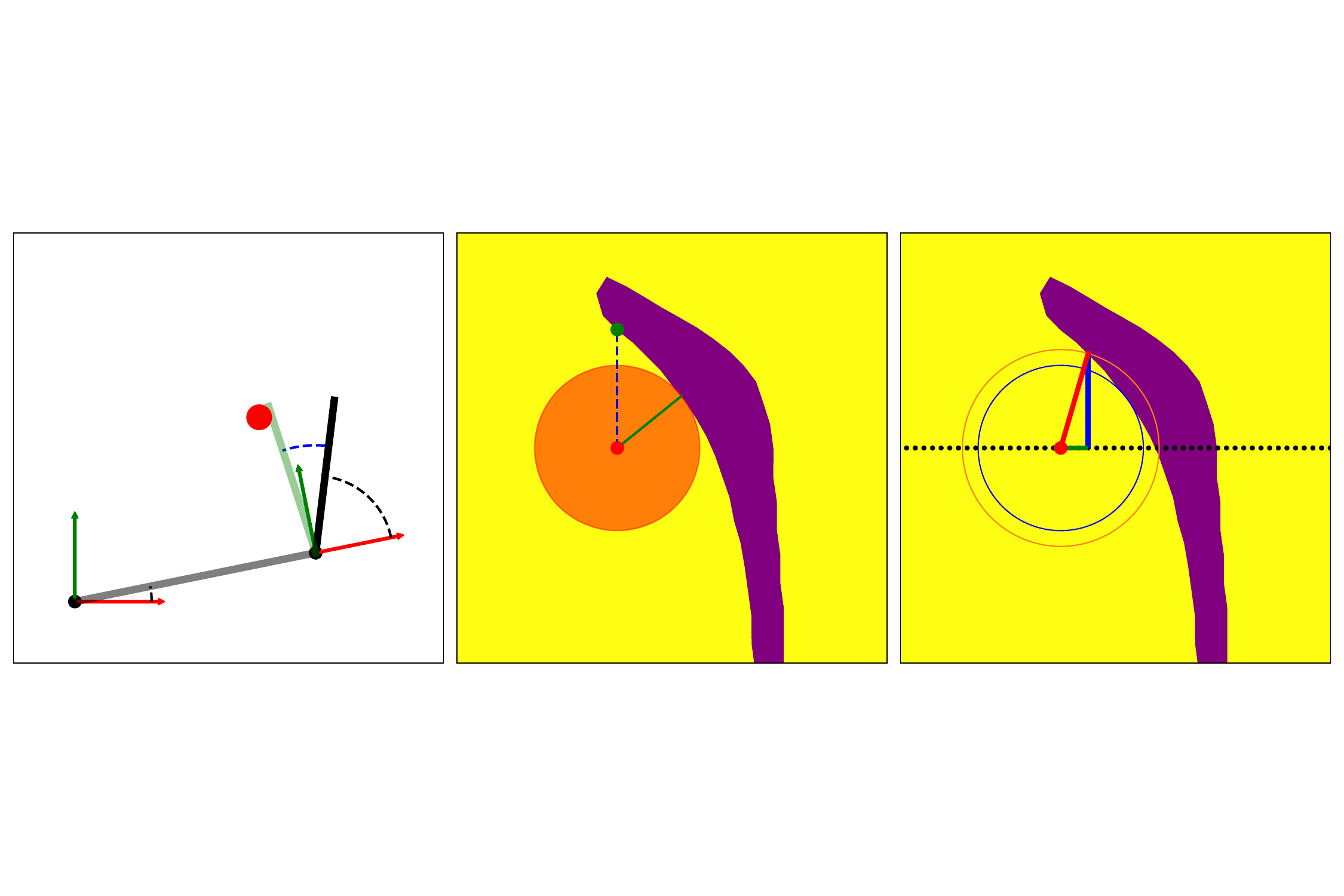}};

\node at (3.9, 3.45) {$\funcNSDwristLocal$};
\node at (0.8, 0.7) {$\{0\}$};
\node at (4.2, 1.5) {$\{w\}$};

\node at (2.3, 1.15) {$q_m$};
\node at (5, 2.5) {$q_w$};

\node at (3.1, 3.2) {$\geom, \geomWC$};

\node at (7.7, 3.8) {$\funcNSDwristLocal$};
\node at (8.7, 3.3) {$\funcSDwrist$};
\node at (7.8, 3) {$\q$};

\def\xOff{0.06}
\def\yOff{0.26}


\def\xxOff{5.85}
\draw[->, very thick] (\xxOff + \xOff, \yOff) -- (\xxOff + \xOff, \yOff + 1);
\draw[->, very thick] (\xxOff + \xOff, \yOff) -- (\xxOff + \xOff + 1, \yOff);

\node at (\xxOff + \xOff + 1.25, \yOff+0.2) {$q_m$};
\node at (\xxOff + \xOff + 0.2, \yOff+1.25) {$q_w$};

\def\xxOff{11.72}

\draw[->, very thick] (\xxOff + \xOff, \yOff) -- (\xxOff + \xOff, \yOff + 1);

\draw[->, very thick] (\xxOff + \xOff, \yOff) -- (\xxOff + \xOff + 1, \yOff);
\node at (\xxOff + \xOff + 1.25, \yOff+0.2) {$q_m$};
\node at (\xxOff + \xOff + 0.2, \yOff+1.25) {$q_w$};

\node at (14.1, 2.9) {$\dNormal$};
\node at (14.5, 3.5) {$\dInPlane$};
\node at (13.8, 3.7) {$d$};
\node at (13.3, 2.55) {$d_{\text{min}}$};
\node at (13.6, 2.9) {$\qSamp$};

\end{tikzpicture}
\caption{
\textbf{All:}
A conceptual illustration of how the wrist configuration distance is computed for a 2 DOF arm, where $q_m$ is the main joint and $q_w$ is the wrist joint. The left plot is illustrated in the world space, the rest is illustrated in the configuration space.
\textbf{Left:}
The manipulator has two links (colored black). The main link is made transparent to emphasize that its collision with the obstacle is not considered. The obstacle is a disk (colored red). Its parameter vector is expressed in the base frame $\{0\}$ as $\geom$ and in the wrist frame $\{w\}$ as $\geomWC$. Computing the distance from the local wrist function results in a wrist movement (dashed blue arc) that moves the wrist to the obstacle boundary (transparent green). The distance or the movement does not include the main joint.
\textbf{Middle:}
Corresponding configuration space representation of the left plot. The configuration (red point); distance computed by the local wrist nSCDF (dashed blue line); the ground truth distance including the movement of the main joint (transparent red disk and green line).
\textbf{Right:}
To compute an estimate of the distance where the movement of the main joint is included, we discretize the space of the main joints (black points along the horizontal axis). For each grid point, we compute a distance, by forming the hypotenuse (red line) with the distance of the main joint (green line) and the wrist distance (blue line). This forms an estimated safe set (orange circle). From all estimates we take the smallest distance (blue circle).
}
\label{fig:method:SCDF_dg}
\end{figure*}
The high-level idea is to approximate the obstacle boundary by intersecting it with $\NrGridPoints$ hyperplanes, resulting in $\NrGridPoints$ slices. To compute a distance $d$ from an arbitrary point to the boundary of the $i$-th slice, we combine the shortest distance to the plane, $d_{\perp}$, and the in-plane distance to the slice boundary, $d_{\parallel}$. To compute the distance to the obstacle boundary, we return the smallest slice distance among the $N$ slices. 
\\\\
The overall idea is simple but assumes that we can efficiently compute the in-plane slice distances. As a first step to enable this, we form the hyperplanes from the subspace spanned by the wrist joints. Thus, translating the hyperplane in the normal direction corresponds to moving the main joints. Since the subspace is reduced in dimensionality, this enables us to use our previous method to learn an nSCDF $\funcNSDwristLocal : \spaceR^3 \times \setGeom \mapsto \spaceR$ for the in-plane distances, illustrated in the left and middle parts of Figure~\ref{fig:method:SCDF_dg}, and it is referred to as a \textit{local} wrist nSCDF. Since the slice direction is along the subspace of the wrist joints, it only considers the manipulation of the wrist links. The obstacle shape is expressed relative to the wrist coordinate system, which we denote as~$\geomWC$. Thus, moving the main joints, i.e., translating the hyperplane in the normal direction, results in a new shape vector and new slice distance.
\\\\
With an efficient representation to compute the in-plane slice distances, we can now compute signed distances to the high-dimensional obstacle boundary of the wrist links. We present our method in Algorithm~\ref{alg:dist} and illustrate the process in the right part of Figure~\ref{fig:method:SCDF_dg}.
\\\\
As input, we require a query configuration $\qSamp=(\qSampMain, \qSampWrist)$ from which a signed distance should be computed, an obstacle shape $\geom$ and a trained local wrist nSCDF $\funcNSDwristLocal$. We start by computing the collision status, $s \in \{-1, 1\}$, of our input query point, line~\ref{alg:dist:collsign}, using a conventional collision checker. Next, we compute the necessary inputs for the slice-based representation in line~\ref{alg:dist:grid1} and \ref{alg:dist:grid2}. That is, we form a uniform grid over the main joints, resulting in $\NrGridPoints$ points, line~\ref{alg:dist:grid1}. The grid points correspond to the offset of the hyperplanes. For each of these grid points, $\q_m^i$, $i \in \intRange_{1:N}$, we compute the pose of the wrist frame, $\matT(\q_m^i)$, and express the position of the obstacle shape in the wrist frame, yielding a local shape vector $\geomWC^i$, line~\ref{alg:dist:grid2}. 
\\\\
We are now ready to compute the requested signed distance, which we illustrate in the right part of Figure~\ref{fig:method:SCDF_dg}. We initialize our estimate of the distance, line~\ref{alg:dist:dmin_start}, and iterate over all the $\NrGridPoints$ grid points, line~\ref{alg:dist:est_start} to~\ref{alg:dist:est_end}. We start by computing the shortest distance to the plane from our given point. Since the hyperplane is oriented along the subspace spanned by the wrist joints, the shortest distance to the hyperplane simply becomes the distance between the main joints, which is
\begin{equation}
    \dNormal = \norm{\qSampMain - \q_m^i}.
\end{equation}
Next, we compute the in-plane distance, line~\ref{alg:dist:wrist_dist}, by utilizing the local wrist nSCDF. We do this according to
\begin{equation}
 \dInPlane = \max(0, \: s \cdot \funcNSDwristLocal(\qSampWrist, \: \geomWC^{i})). 
\end{equation}
In the above expression, we only include wrist distances that have the same collision sign as the query point. The reason for this is that if we move into a region where there is a change in collision status, then the local wrist nSCDF will add extra distance to our estimate, making it too long. As a remedy, we ignore the in-plane distance if there is a sign change, i.e., the product $s \cdot \funcNSDwristLocal(\qSampWrist, \: \geomWC^{i})$ is negative. We then combine the distances, line~\ref{alg:dist:hypot}, and store only the smallest distance, line~\ref{alg:dist:min}. Once we have iterated over all the points, we return the signed distance, line~\ref{alg:dist:return}. 
\\\\
Algorithm~\ref{alg:dist} can be used to compute a labeled dataset, from which we can train a neural network in a supervised manner akin to what was done in~\cite{pbrm}.
\begin{algorithm}[H] 
\centering
\caption{Computes an approximate signed configuration distance from a high-dimensional surface through a sliced-based representation.}
\label{alg:dist}
\begin{algorithmic}[1]
    \Require $\qSamp = (\qSampMain, \qSampWrist)$, $\geom$, $\funcNSDwristLocal$
    \State $s$ $\gets$ compute collision-sign $\funcColl(\qSamp)$ \label{alg:dist:collsign}
    \State $(\q_m^{1}, \hdots, \q_m^{N})$ $\gets$ form grid over main joints \label{alg:dist:grid1}
    \State $(\geomWC^{1}, \hdots, \geomWC^{N})$ $\gets$ compute local shapes \label{alg:dist:grid2}
    \State $\dMin$ $\gets$ $\infty$ \label{alg:dist:dmin_start}
    \For{$i \in \intRange_{1:N}$} \label{alg:dist:est_start}
        \State $\dNormal$ $\gets$ $\norm{\qSampMain - \q_m^{i}}$ \label{alg:dist:main_dist}
        \State $\dInPlane$ $\gets$ $\max(0, \: s \cdot \funcNSDwristLocal(\qSampWrist, \: \geomWC^{i}))$ \label{alg:dist:wrist_dist}
        \State $d \gets \sqrt{\dInPlane^2 + \dNormal^2}$ \label{alg:dist:hypot}
        \State $\dMin$ $\gets$ $\text{min}(\dMin, d)$ \label{alg:dist:min}
    \EndFor \label{alg:dist:est_end}
    \State $\Return$ $s \cdot \dMin$ \label{alg:dist:return}
\end{algorithmic}
\end{algorithm}
\section{Funnel planning}
\def\qStart{\q_{\text{s}}}
\def\qGoal{\q_{\text{g}}}
In the following we propose a computationally efficient method to plan for a collision-free funnel where we balance path length and clearance. The control aspect will be addressed in Section~\ref{sec:funnel_control}. We assume access to a differentiable SCDF $\funcSD$. In the experiments, Section~\ref{sec:exp:FP}, we will use the nSCDF as an approximation of the SCDF. We represent the funnel by a trajectory, $\funcTraj(t) : [0, \timeFinal] \mapsto \setC$ and a funnel function, $\funcFunnel(t) : [0, \timeFinal]\mapsto\spaceR_{++}$. Since we want guarantees on collision-free motion, we use the SCDF as a funnel function $\funcFunnel(t) = \funcSD(\funcTraj(t))$. As commonly done in motion planning, we decouple our trajectory into a path $\funcTraj(s) : [0,1] \mapsto \setC$ and a time-scaling function $s(t):[0, \timeFinal] \mapsto [0, 1]$. Computing a time-scaling function can be done very efficiently \cite{toppra}, while computing the path is nontrivial. Thus, in the following, we focus on the path planning aspect. Our method addresses the following optimization problem
\begin{subequations}
\label{eq:fop}
\begin{align}
\underset{\funcTraj}{\operatorname{min}}
    \quad 
    & 
    \int_{0}^{1} \frac{1}{\funcSD(\funcTraj(s))}
    + 
    \norm{ \funcTraj^{\prime}(s) } 
    \: 
    ds
    \label{eq:fop:obj} \\
    \textrm{s.t.} \quad & \funcTraj(0) = \qStart, \funcTraj(1) = \qGoal, \label{eq:fop:start_goal} \\
    & \funcSD(\funcTraj(s)) > 0, \forall s \in [0, 1]. \label{eq:fop:cf}
\end{align}
\end{subequations}
The objective \eqref{eq:fop:obj} of the above optimization problem is to produce a path that balances clearance and path length, with the requirement that the path should start and end in the given inputs \eqref{eq:fop:start_goal}, and that the path is collision-free~\eqref{eq:fop:cf}.
\\\\
The optimization problem is nonconvex and hard to solve. Faced with this difficulty, we approach it heuristically, dividing the problem into simpler subproblems, which we illustrate in Figure \ref{fig:method:pipeline}. First, we produce a feasible path with respect to~\eqref{eq:fop}, using, e.g., a sampling-based planner, and thereby decide on a certain homotopy class. Next, we optimize its clearance, i.e., the first part of the objective~\eqref{eq:fop}, through a projection step. This is illustrated in the left part of Figure~\ref{fig:method:pipeline} and discussed in Section~\ref{sec:fp:map}. This results in a path with high clearance but is most likely suboptimal with respect to path length. To balance path length, we formulate a weighted shortest path (WSP) problem, discussed in Section~\ref{sec:fp:wsp} and shown in the middle part of Figure \ref{fig:method:pipeline}. However, a natural consequence of this step is that the resulting path becomes nonsmooth. To remedy this, we solve a convex optimization problem, where the path is smoothed and constrained to be close to the original optimized path, discussed in Section \ref{sec:fp:scp} and visualized in the right part of Figure \ref{fig:method:pipeline}. The resulting optimized path is then time-scaled to produce our required trajectory.
\begin{figure*}[t!]
\centering
\includegraphics[height=.35\textwidth, trim={0.25cm 6cm 0.25cm 6cm},clip]{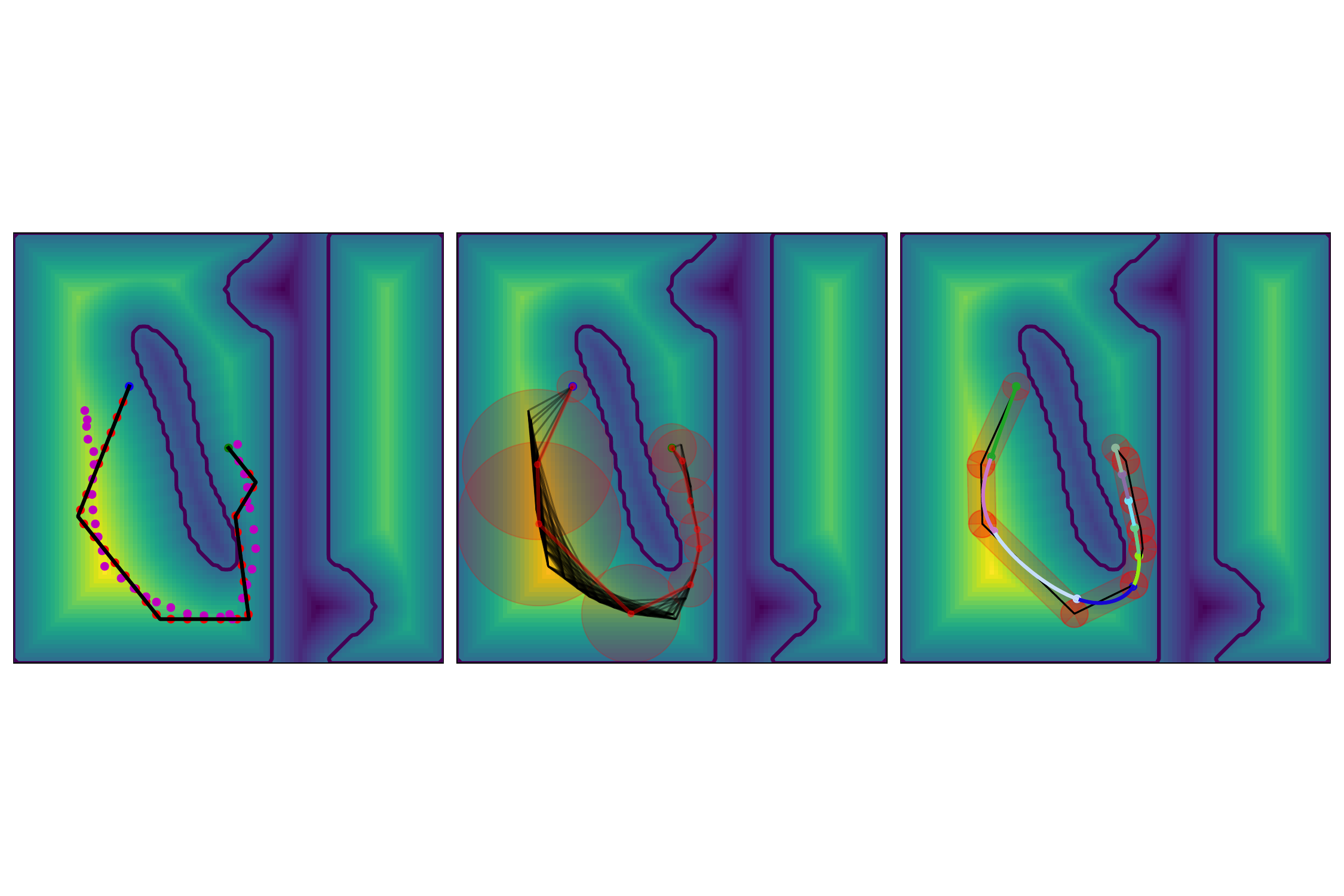}
\caption{ 
Illustrations in the configuration space. The SCDF, $\funcSD$, is illustrated as a colormap, with darker colors indicating lower distances. The zero-level contour, i.e., $\partial \setCo$, is represented by the solid purple curve.
\textbf{Left:}
Illustration of our medial-axis projection step. A discretized path (red points) is projected to points approximately on the medial axis (purple points), resulting in increased clearance.
\textbf{Middle:}
The projection steps result in a nonsmooth and distorted path. We solve a weighted shortest path problem by constructing a graph by using the points on the path as vertices. The edges are formed over intersecting balls (black lines). The weight function balances path length and clearance. Solving the problem results in an optimized path in the graph (red curve), which is continuously collision-free due to the intersection requirement of the resulting balls (red circles).
\textbf{Right:}
We smoothen the piece-wise linear path by solving a convex optimization problem. The key idea is to find a smooth path that is close to the original one (black curve). We do this by constructing collision-free capsules (transparent red) along the connecting line segments of the path. Each capsule is assigned a B\'ezier curve (colored curve within capsule). Continuity constraints are enforced with its neighboring curve. The closeness of the original path is controlled by the width of the capsules. Solving the optimization problem results in a smooth collision-free path.
}
\label{fig:method:pipeline}
\end{figure*}
\subsection{Approximate medial-axis projection}
\label{sec:fp:map}
\def\funcPathNonOpt{\tilde{\bm{r}}}
\def\funcPathMaxClear{\bm{r}}
\def\pathSBP{\text{P}}
\def\pathSBPOpt{\tilde{\text{P}}}
\def\pathSP{\bar{\text{P}}}

\def\qCl{\tilde{\q}}
\def\qSP{\bar{\q}}
\def\tupleGraph{\text{G}}
\def\setV{\mathcal{V}}
\def\setE{\mathcal{E}}
\def\v{\bm{v}}
\def\NrDiscP{m}
\def\NrSPP{m}

\def\tupleCorr{\text{C}}
\def\tupleCorrProj{\tilde{\text{C}}}
\def\tupleCorrWSP{\text{C}^{\star}}

\def\ballCenter{\bm{c}}
\def\ballCenterProj{\tilde{\bm{c}}}
\def\ballCenterWSP{\bm{c}^{\star}}
\def\setBproj{\tilde{\setB}}
\def\NrCorBalls{m}
\def\funcNeigh{\bm{b}}
\def\funcRay{\bm{r}}
We assume the availability of a feasible path, connecting the query points $\qStart$ and $\qGoal$. This path is then discretized into waypoints resulting in the sequence
\begin{equation}
    \tupleCorr = [\ballCenter_1, \hdots, \ballCenter_{\NrCorBalls}].
\end{equation}
Next, we want to increase the clearance of the path. A suitable candidate to achieve this is given by the medial axis, defined as 
\begin{equation}
    \label{eq:medial_axis}
    \setMA = \{ \q  \in \setC \: | \: \funcCard(\setN(\q)) > 1 \},
\end{equation}
where $\setN(\q)$ defines the inscribed boundary points
\begin{equation}
    \setN(\q) = \{ \q_{b} \in \partial\setCo \: | \:  \norm{\q - \q_{b}} = \funcSD(\q) \}.
\end{equation}
Thus, the medial axis corresponds to the locus of points with locally maximal clearance. We project our path onto the medial axis in a greedy fashion. First, we define the closest boundary point function as 
\begin{equation}
    \funcNeigh(\q) = \q -  \funcSD(\q) \nabla_{\q} \funcSD.
\end{equation}
The above function simply computes a boundary point by scaling the negative gradient with the corresponding distance. For each point on the path,~$\ballCenter_i \in \tupleCorr$, we trace a ray
\begin{equation}
    \funcRay(s) = \ballCenter_i + \nabla_{\q} \funcSD \cdot s, \: s\ge0.
\end{equation}
Then, we search for a narrow interval $[s_l, s_u]$ defined as 
\begin{align}
    \begin{aligned}
        & s_l \ge 0, s_u > s_l, s_u - s_l \le \varepsilon, \\
        & \funcNeigh(\ballCenter_i) = \funcNeigh(\funcRay(s_l)) \neq \funcNeigh(\funcRay(s_u)),
    \end{aligned}
\end{align}
where $\varepsilon$ is a small user-defined constant that controls the width of the interval. The above expresses the fact that there is a change in the boundary point at the upper end of the interval. To form an approximate medial axis point, i.e., a point where two boundary points are approximately equidistant, we first compute a narrow interval through bisection. Then we form the center point of the interval, that is
\begin{equation}
    \ballCenterProj_i = \funcRay((s_l+s_u)/2).
\end{equation}
Repeating the process for all points results in a path with increased clearance, which  we denote by
\begin{equation}
    \tupleCorrProj=[\ballCenterProj_1, \hdots, \ballCenterProj_m].
\end{equation}
\subsection{Weighted shortest path problem}
\label{sec:fp:wsp}
\def\funcWeight{\kappa}
The previous step results in a path with increased clearance, but it potentially includes redundant motion. To handle this, we formulate a WSP problem that balances both clearance and path length. We derive a weighted graph, $\tupleGraph=(\setV, \setE, \funcWeight)$, where the vertices, $\setV$ are formed from the projected points, $\setV=\{\ballCenterProj_i \}_{i=1}^{m}$. Two vertices are connected by an edge if their balls \eqref{eq:ball} intersect, that is
\begin{equation}
\setE = \{ (\ballCenterProj_i, \ballCenterProj_j) \in \setV^2 \: | \: \setB(\ballCenterProj_i) \cap \setB(\ballCenterProj_j) \neq \emptyset  \}.
\end{equation}
Finally, we define the weight function $\funcWeight: \setC \times \setC \mapsto \spaceR_{+}$ as
\begin{equation}
    \kappa(\c_u, \c_v)  = 1 / \min(\funcSD(\c_u), \funcSD(\c_v)) + \norm{\c_u - \c_v},
\end{equation}
which is a discretized approximation of our original objective~\eqref{eq:fop}. Having defined the graph, we now solve for the shortest weighted path, e.g., using A$^\star$, resulting in an optimized path in the graph according to 
\begin{equation}
    \tupleCorrWSP = [\ballCenterWSP_1, \hdots, \ballCenterWSP_m].
\end{equation}
\subsection{Path smoothing through convex optimization}
\label{sec:fp:scp}
\def\funcBZ{\bm{\beta}}
\def\NrSegments{m-1}
\def\setCAP{\mathcal{CAP}}
\def\NrBzCntrl{d}
\def\capRadial{r^{\text{cap}}}
\def\setLine{\mathcal{LS}}
\noindent
The resulting optimized path is nonsmooth and needs to be smoothened in order to be successfully tracked. Naturally, we also want the smoothened path to be close to our optimized path, such that it retains its optimized features. 
\\\\
First, we express a narrow region around the optimized path. For each neighboring pair, $(\ballCenterWSP_i, \ballCenterWSP_{i+1}), i \in \intRange_{1:\NrSPP-1}$, in the path, we form a line segment as
\begin{equation}
    \setLine_i = \{ (1-s) \cdot \ballCenterWSP_{i}  + s \cdot \ballCenterWSP_{i+1} \: | \: s \in [0, 1] \}.
\end{equation}
Our requirement of being close to the path can be translated as being within a radial distance $\capRadial_i \in \spaceR_+$ along each line-segment. This is the same as being within a capsule, which we form along each line-segment according to
\begin{equation}
\setCAP_i =   \setLine_i \oplus \{ \q \in \spaceR^{\nrDOF} \: | \: \norm{\q} < \capRadial_i \}, i \in \intRange_{1:\NrSPP-1}.
\end{equation}
\def\setS{\mathcal{S}}
\def\rInt{r^{\cap}}
\def\cInt{\c^{\cap}}
In the above, the operator $\oplus$ is the Minkowski sum. The resulting set is convex. To further guarantee that the capsule is collision-free, i.e., $\setCAP_i \subset \setCf$, we first define the intersecting sphere, $\setS$, between the boundaries of the neighboring balls. We define it as
\begin{equation}
 \setS(\cInt_i, \rInt_i) = \partial \setB(\ballCenterWSP_i) \cap \partial\setB(\ballCenterWSP_{i+1}), i \in \intRange_{1:\NrSPP-1},
\end{equation}
where $\cInt_i$ and $\rInt_i$ are the center point and radius, respectively. We then require that the capsule radius is smaller than the intersection, i.e., $\capRadial_i \le \rInt_i$. This guarantees that the capsule is inside the intersection and thus becomes collision-free. 
\\\\
Having expressed a suitable region that our smoothened path should stay within, we next want a suitable family of parameterized smooth curves, which are amenable to numerical optimization. For this, we choose B\'ezier curves, $\funcBZ : [0,1] \mapsto \spaceR^{\nrDOF}$, which are $d$-times differentiable, denoting its $k$-th derivative as $\funcBZ^{(k)}$. We now formulate the optimization problem to obtain a smooth path, which is
\begin{subequations}
\label{eq:smth}
\begin{align}
\underset{\funcBZ_1, \hdots, \funcBZ_{m-1}}{\operatorname{min}} \quad
&
\sum_{i=1}^{\NrSegments} \int_0^1 ||\funcBZ_i^{(2)}(s) ||^2 ds  \label{eq:smth:obj}
\\
    \textrm{s.t.} \quad & \funcBZ_i(s) \in \setCAP_i, \forall s \in [0, 1], &  \label{eq:smth:closeness} \\
                  \quad & \funcBZ_1(0) = \qStart, \: \funcBZ_{\NrSegments}(1) = \qGoal, &   \label{eq:smth:start_goal} \\
                  \quad & \funcBZ_j^{(k)}(1) = \funcBZ_{j+1}^{(k)}(0), &  \label{eq:smth:cont} \\
                  \quad & i \in \intRange_{1:\NrSegments}, \: j \in \intRange_{1:m-2}, k \in \intRange_{0:\NrBzCntrl}.  & \nonumber
\end{align}
\end{subequations}
In the above optimization problem, we have $\NrSegments$ B\'ezier curves as our decision variables. The objective \eqref{eq:smth:obj} is to yield a smooth curve, which we express by penalizing the second derivative. Collision-free guarantees and closeness to the original path are expressed by constraining the curves to lie in their corresponding capsules \eqref{eq:smth:closeness}. In $\eqref{eq:smth:start_goal}$ we enforce that the curves start and end in the given inputs. Finally, in \eqref{eq:smth:cont}, we enforce continuity between the neighboring curves. The optimization problem is convex, thus, we can solve it fast and reliably. Solving the optimization problem results in our path, which we time-scale, yielding our desired trajectory.
\section{Funnel control}
\label{sec:funnel_control}
This section derives a funnel control protocol that guarantees the tracking of $\funcTraj(t)$ while staying within the funnel, yielding predefined transient and steady-state performance. As in the previous section, we assume access to a differentiable SCDF $\funcSD$. In the experiments, Section~\ref{sec:exp:FP}, we will use the nSCDF as an approximation of the SCDF. \\\\
We define the error function as 
\begin{equation}
    \funcErr = (e_1,\dots,e_n) = \q - \funcTraj(t).
\end{equation}
and choose the funnel function as $\funcFunnel:[0,t_f]\mapsto\spaceR_{++}$, with
\begin{equation}
    \funcFunnel(t) = \funcSD(\funcTraj(t)).
\end{equation} 
Note that, for all $t\in[0, \timeFinal]$, since $\funcTraj(t) \in \setCf$, i.e., the trajectory is collision-free (see previous section), it holds that $\funcFunnel(t) > 0$. Furthermore, in view of the definition of $\funcNSD(\cdot)$ in Section~\ref{sec:scdf:neural_scdf}, $\funcFunnel(t)$ is continuous and continuous differentiable for almost all $t \in [0,\timeFinal]$. The control design aims to establish
\begin{equation} \label{eq:PPC distance}
    0 < \norm{\funcErr(t)} < \funcFunnel(t), \quad \forall t \in [0, \timeFinal], 
\end{equation}
which would guarantee collision-free navigation (see Figure~\ref{fig:intro}). Note that \eqref{eq:PPC distance} is equivalent to staying in the spherical collision-free regions $\setB$ defined in \eqref{eq:ball}. 
\\\\
To achieve \eqref{eq:PPC distance}, we employ a model-based version of the funnel control algorithm developed in \cite{kdf,verginis2019robust}. In particular, we define the normalized and transformed errors:
\begin{subequations} \label{eq:ksi and epsilon}
\begin{align} 
   \xi &= \frac{\norm{\funcErr}^2}{\funcFunnel(t)^2}, \label{eq:ksi}  \\
   {\varepsilon}(\xi) &= \ln\left( \frac{1}{1-\xi} \right). \label{eq:epsilon}
\end{align}
\end{subequations}
The control algorithm then aims to retain the boundedness of $\varepsilon(t)$, which guarantees that $|\xi(t)| < 1$ and hence \eqref{eq:PPC distance} is satisfied for all $t\in[0, \timeFinal]$. We next define a reference velocity profile as 
\begin{align*}
    \bm{v}_r(\xi,\funcErr) = - k_1 J_{\varepsilon}(\xi) \funcErr,
\end{align*}
where $J_{\varepsilon}(\xi) = \frac{\textup{d}\varepsilon(\xi)}{\textup{d}\xi} = \frac{1}{1-\xi} $, and $k_1$ is a positive gain. 
Subsequently, we follow a backstepping-like method and define the velocity errors 
\begin{align*}
    \bm{e}_v = (e_{v_1},\dots,e_{v_n})   = \qdot - \bm{v}_r(\xi,\funcErr),
\end{align*}
as well as positive and bounded functions $\rho_{v_i}(t)$ satisfying $ \rho_{v_i}(0) > |e_{v_i}(0)|$ for all $i \in \intRange_{1:\nrDOF}$. We define   
\begin{align*}
   \bm{\xi}_v &= (\xi_{v_1},\hdots,\xi_{v_n}) = \bm{\rho}_v(t)^{-1} \bm{e}_v, \\
   \bm{\varepsilon}_v(\bm{\xi}_v) &= (\varepsilon_{v_1},\hdots,\varepsilon_{v_n}),
\end{align*}
with $\varepsilon_{v_i} = \ln\left( \frac{1+\xi_{v_i}}{1-\xi_{v_i}} \right)$, $i\in \intRange_{1:\nrDOF}$, and $\bm{\rho}_v(t) = \funcDiag((\rho_{v_1}(t), \hdots, \rho_{v_n}(t)))$. We then design the model-based control law as 
\begin{align} \label{eq:control law}
    \u &=  \hat{\bm{C}}(\q,\qdot)\qdot + \hat{\bm{g}}(\q)
    - 
    \funnelVelocityGain \bm{\rho}_v(t)^{-1} \bm{J}_v(\bm{\xi}_v) \bm{\varepsilon_v}(\bm{\xi}_v), 
\end{align}
where $\bm{J}_v(\bm{\xi}_v) = \funcDiag(J_{v_i}(\xi_{v_i}) )$ and $J_{v_i}(\xi_{i_ v}) = \frac{\textup{d} \varepsilon_{v_i}  }{\textup{d} \xi_{v_i} } = \frac{ 2 }{1 - \xi_{v_i}^2 }$ for $i\in \intRange_{1:\nrDOF}$ and $k_2$ is a positive gain. The guarantees of the proposed control algorithm are given in the following theorem, whose proof can be found in Appendix~\ref{app:proof}.
\begin{theorem} \label{thm:ppc}
Consider a robot manipulator subject to the dynamics \eqref{eq:dyn_snd} and satisfying $ \norm{\funcErr(0)} < \funcFunnel(0)$ for a reference trajectory $\funcTraj(t)$ defined for $t \in [0, \timeFinal]$. The controller \eqref{eq:control law} guarantees that $ \norm{\funcErr(t)} < \funcFunnel(t)$ for all $t \in [0, \timeFinal]$.
\end{theorem}

\section{Numerical experiments}
\begin{figure*}[t!]
    \centering
    \begin{tikzpicture}
        \node[anchor=south west,inner sep=0] at (0,0) {\includegraphics[width=0.45\textwidth, trim={0cm 0cm 0cm 2cm}, clip]{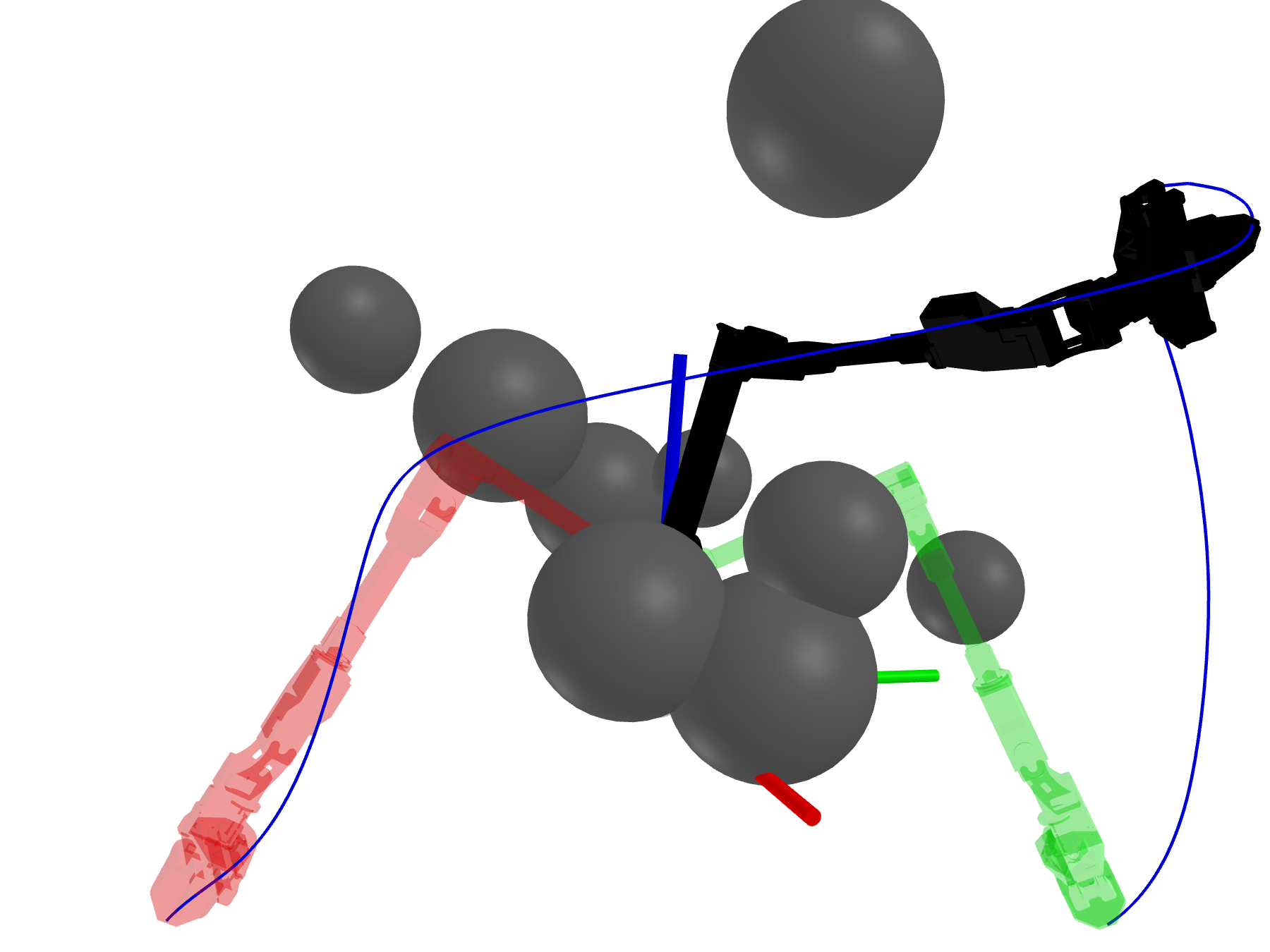}};
        \node[anchor=south west,inner sep=0] at (0.5\textwidth,0) {\includegraphics[width=0.45\textwidth, trim={2cm 0 2cm 2cm}, clip]{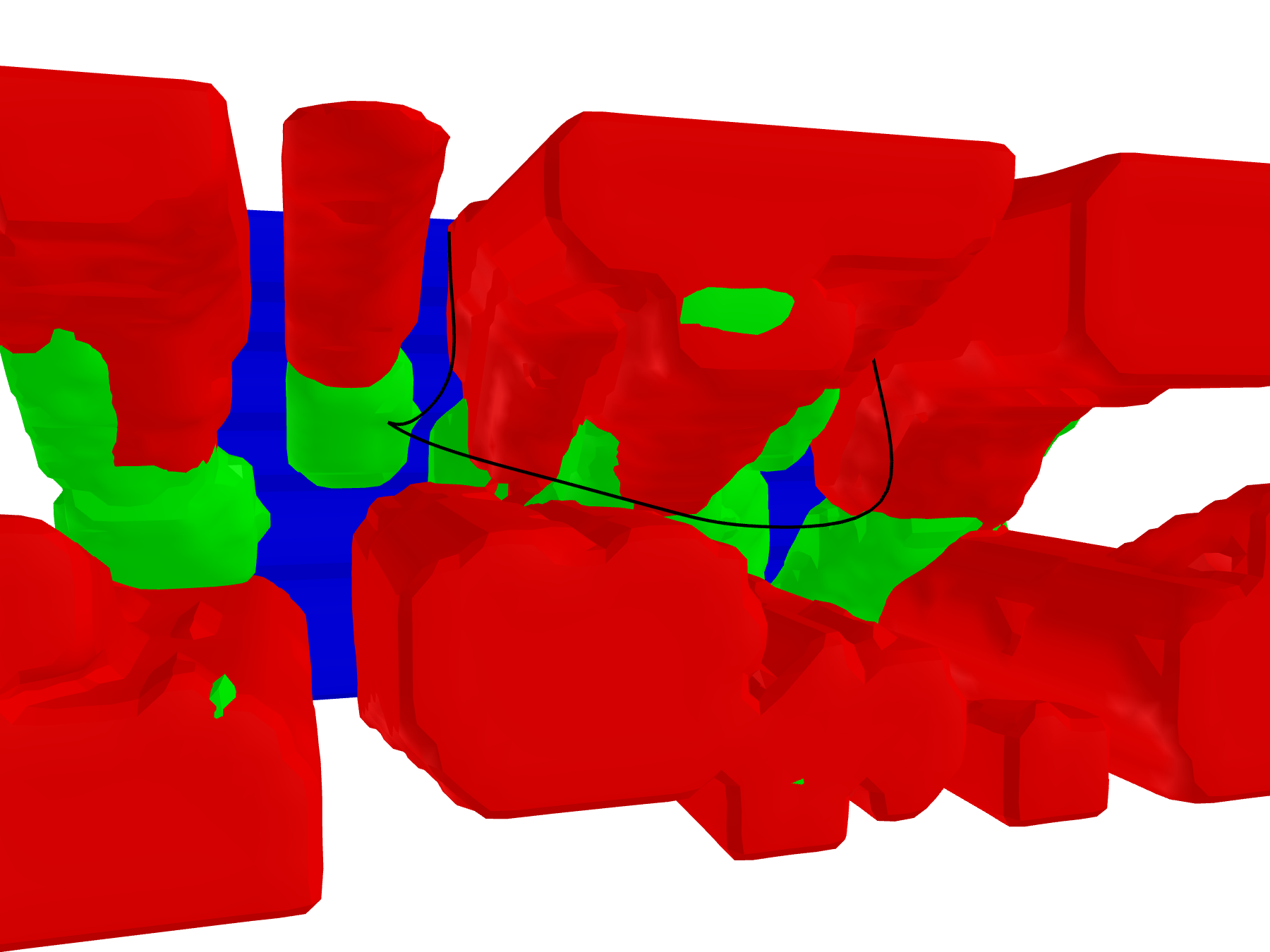}};
    \end{tikzpicture}
    \caption{
    \textbf{Left:} Illustration in world space, showing an example problem instance used in the experiments. The manipulator is 6 DOF colored black at a configuration along a collision-free path. The end-effector position is traced out along the path (blue curve). The start and goal configurations are colored transparent red and green, respectively. The obstacles are spheres (colored gray).
    \textbf{Right:} The learned configuration space representation of the world instance to the left, illustrated in the sub-space of the first three joints, i.e., main joints. The red, green and blue colored regions are the obstacle spaces for the main joints, wrist joints and self-collisions. The resulting path from our method is the black curve.
    }
    \label{fig:exp:demo}
\end{figure*}
\def\funcNSDsc{\funcNSD^{sc}}
The following section presents the numerical experiments conducted with a ViperX-300 \cite{viper_robot} 6 DOF robot, $\nrDOF=6$, which is illustrated in the left part of Figure~\ref{fig:exp:demo}.
\\\\
In the following we approximate the SCDF with the nSCDF, i.e., $\funcSD\approx\funcNSD$. For our considered robot, we also have to include an nSCDF for self-collisions. Thus, we define the nSCDF according to
\begin{equation}
    \label{eq:nSCDF_complete}
    \funcNSD(\q, \geom) = \min(\funcNSDmain(\q_m, \geom), \funcNSDwrist(\q_m, \q_w, \geom), \funcNSDsc(\q)),
\end{equation}
where $\funcNSDsc$ is the nSCDF for self-collisions.
\\\\
In Section \ref{sec:exp:SCDF}, we present how we learn and validate the nSCDF in \eqref{eq:nSCDF_complete}. In Section \ref{sec:exp:FP}, we focus on funnel planning. We conclude with Section~\ref{sec:exp:FC}, where we show the benefits of using the inherit funnel provided by the nSCDF for funnel control. In particular, we show that we get better motion performance, i.e., lower jerk, and robustness properties, i.e., we can withstand larger disturbances.
%

\subsection{Neural representation}
\label{sec:exp:SCDF}
\subsubsection{Obstacle representation}
As a first step to learn an nSCDF, we have to decide on a parameterized obstacle representation. We use spheres, as they are flexible in their representation of generic obstacles \cite{hubbard_spheres}. Thus, our geometry vector is defined as
\begin{equation}
    \geom = (\p, r) \in \setW \times \spaceR_{++},
\end{equation}
where $\p$ defines the center point of the sphere and $r$ its radius.
\subsubsection{Symmetries}
\def\RotMat{\bm{R}}
The first joint of our manipulator has a rotation around the $z$-axis. If we rotate both the manipulator and the obstacle by $\Delta \theta$ around the $z$-axis, then the world space distance does not change. This fact also applies to the nSCDF, which we express as
\begin{equation}
    \funcNSD(\q, (\p, r)) = \funcNSD((\theta_1 + \Delta\theta, \q_{2:6}), (\RotMat_z(\Delta\theta) \p, r)).
\end{equation}
In the above, $\RotMat_z(\theta)$, represents the rotation matrix around the $z$-axis. We exploit this relationship in the following data generation process, by only generating obstacles in the $xz$-plane. Hence, when representing the position of the obstacle, we only specify the $xz$-components. This reduces our input dimension and simplifies the learning.
\subsubsection{Data generation and learning}
The data generation and learning process for $\funcNSDmain$, $\funcNSDwristLocal$ and $\funcNSDsc$ each followed \cite{pbrm}. For the wrist distance function, $\funcNSDwrist$, we discretized each main joint into 50 equally spaced points between the joint limits, resulting in $N=50^3$ grid points. We then generated 2000 random obstacles in the world space and used Algorithm~\ref{alg:dist} to compute the signed configuration distances. When sampling the configurations, we biased the sampling according to the following configuration distance intervals.
\begin{itemize}
    \item \textbf{Zero crossing interval:} $[-0.05, 0.05] \: rad$.
    \item \textbf{Operating interval:} $[0.05, 1.0]\:rad$.
\end{itemize}
The sampling was done with proportions $1/3$ and $2/3$, respectively. The first interval (zero crossing) is motivated by our need to use the nSCDF as a collision detector. This requires a good description of when the sign change occurs. The second interval is selected based on our observation that once the nSCDF is used with multiple obstacles, the min operator \eqref{eq:scdf:min} filters away large distances, making it unnecessary to learn large distances well. Thus, the operating interval is the region in which learning an accurate representation is important. For each generated obstacle, we sampled 3000 configurations according to the specified intervals and proportions. The rest of the learning followed \cite{pbrm}.
\subsubsection{Learning validation setup}
Having trained all the networks required to form the nSCDF in \eqref{eq:nSCDF_complete}, we next validated its ability to form collision-free balls, i.e., property \eqref{eq:ball}. We first generated 100 unseen obstacles. For each obstacle, we sampled 100 configurations that had predicted distances inside the interval $[0.0, 1.0] \: rad$. For each sampled point, we sampled another $10^4$ points within the corresponding ball. The points sampled within the ball were used to certify the ball's collision-free volume by collision checking all the points with a conventional collision checker~\cite{trimesh} using the ground-truth geometries for the obstacle. We then computed the fraction of collision-free points within the ball. Ideally, all points should be collision-free. If, however, there were any collisions, then we recomputed the distance by taking the smallest distance from the center point to all points in collisions, which results in a smaller radius.
\subsubsection{Learning validation results}
From the experiments, we observed that $6\%$ of the balls contained points in collisions. The statistics of the overestimation errors are reported in Table~\ref{tab:exp:ball_ver}.
\begin{table}[t]
\centering
\caption{
Statistics of the overestimation in the ball radius. In total, $10^4$ balls were used. A percentile is denoted with a $Q_p$, where $p$ is the used percentile. All values are in radians.
}
\label{tab:exp:ball_ver}
\begin{tabular}{cccccc}
mean &
$Q_{95}$ &
$Q_{99}$ &
$Q_{99.9}$ &
max
\\
\hline
0.001 & 0.005  & 0.033  & 0.061  & 0.141  \\
\end{tabular}
\end{table}
Observing the results, we notice that the average error is low. The upper percentiles indicate that spread is limited, with $99.9\%$ of the balls having an error that is less than $0.061\:rad$. To account for the possibility of overestimation, we use a conservative safety distance margin of $0.061\:rad$ in the following experiments.
\subsection{Funnel planning}
\label{sec:exp:FP}
\begin{figure*}[t!]
    \centering
    \begin{tikzpicture}
        \node[anchor=south west,inner sep=0] at (0.0, 0.0) {\includegraphics[height=4.5cm]{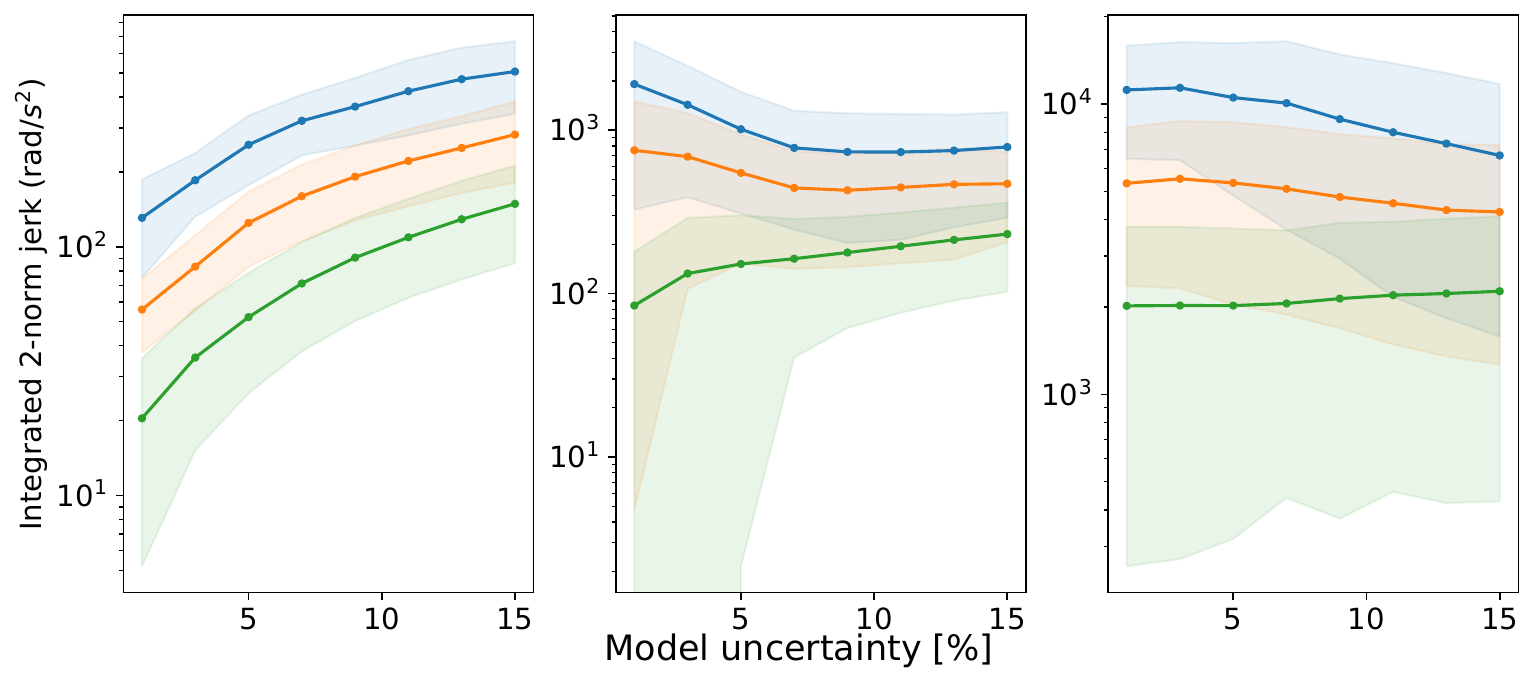}};
        \node[draw] at (2.5, 1) {$\alpha=10^{-3}$};
        \node[draw] at (5.75, 1) {$\alpha=10^{-2}$};
        \node[draw] at (9, 1) {$\alpha=10^{-1}$};
        \node[anchor=south west,inner sep=0] at (11, 0.0) {\includegraphics[height=4.5cm]{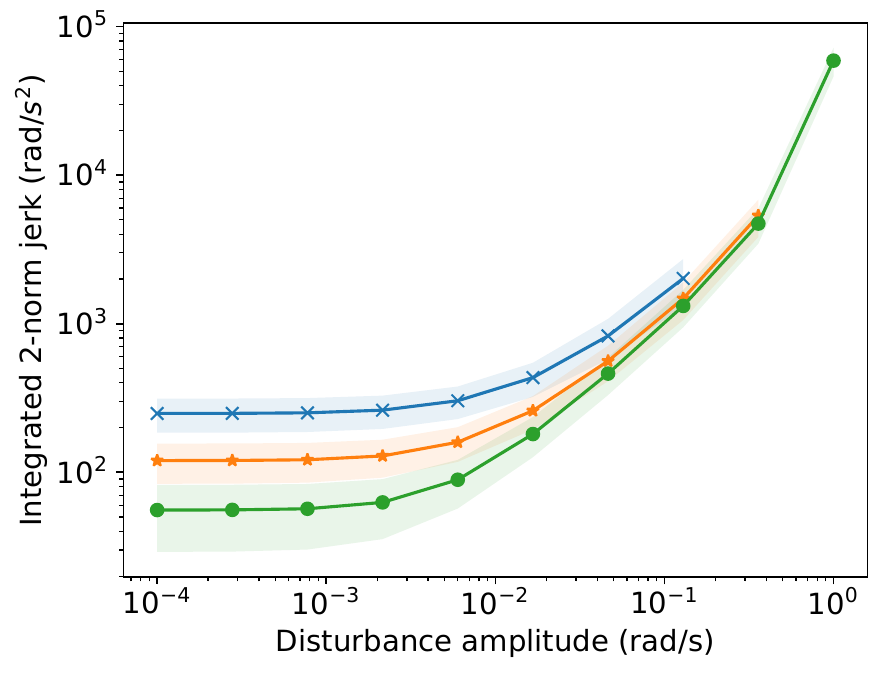}};
    \end{tikzpicture}
    \caption{
    \textbf{All:}
    The vertical axis presents the integrated norm of the resulting jerk. Each data point is the average from 10 simulations in different funnels. If a run failed, then the average was not formed and no data point was added to the graph. The shaded region corresponds to $\pm$ one standard deviation. Funnel sizes: $0.10\:rad$~(blue), $0.15\:rad$~(orange) and with the nSCDF~(green).
    \textbf{Left:}
    The horizontal axis represents different levels of model uncertainties.
    \textbf{Right:}
    The horizontal axis represents different amplitudes on the disturbance signal.
    }
    \label{fig:exp}
\end{figure*}
\noindent
\subsubsection{Setup}
We generate 10 random world instances, each of which contains 10 random obstacles in the world space. At this point, no re-training is performed, and the obstacles were not seen during training. An example world instance is given in the left part of Figure \ref{fig:exp:demo}. For each world instance, 10 random start and goal pairs that cannot be connected with a collision-free straight line were sampled. In total, 100 problem instances were used for evaluation. The performance is evaluated based on path length, clearance and time to compute a path. Clearance is defined as the average distance along the path, formed by discretizing the path evenly into 100 points.
\subsubsection{Planners}
We use the nSCDF for all stages of our presented method. To make sure that we do not plan a funnel that extends beyond the joint limits, we additionally include the signed distance to the joint-limit boundaries when evaluation \eqref{eq:nSCDF_complete}. As a benchmark, we use the Extended Free Space Sampling (EFSS) method from \cite{kdf}. In particular, we use a fixed ball size of $0.1\:rad$, which is validated by $100$ samples using a conventional collision checker \cite{trimesh}. To find a collision-free path, we use a bi-directional RRT planner~\cite{rrt_connect}. All paths from both methods are verified by using a conventional collision checker.
\subsubsection{Results}
The results are presented in Table \ref{tab:exp:pp}.
\begin{table}[t]
\centering
\caption{Statistics from 100 problem instances. Performance is measured based on runtime (R.time), path length (P.len) and clearance (Clear.). The format of the data is average $\pm$ standard deviation. The EFSS method is not able to give any clearance beyond the radius used to validate the region. All paths were collision-free.
}
\label{tab:exp:pp}
\begin{tabular}{l|ccc}
\textbf{Path planning} & R.time $(ms)$ & P.len  $(rad)$ & Clear.  $(rad)$ \\ 
\hline

RRT + EFSS \cite{kdf} & $ 2006 \pm  5545$ & $ 11.60 \pm  4$ & $ 0.10 \pm  0.00$ \\
RRT + nSCDF (ours) & $ 23 \pm  162$ & $ 10.27 \pm  3$ & $ 0.36 \pm  0.08$ \\
\textbf{Method steps} & \multicolumn{3}{l}{} \\ 
\hline
Projection & $ 618 \pm  197$ & $ 11.59 \pm  3$ & $ 0.52 \pm  0.10$ \\
WSP & $ 3 \pm  1$ & $ 9.65 \pm  3$ & $ 0.50 \pm  0.10$ \\
Smoothing & $ 167 \pm  146$ & $ 9.43 \pm  3$ & $ 0.49 \pm  0.10$ \\
\end{tabular}
\end{table}
All produced paths were collision-free. We start by comparing the different path planning approaches. Compared to EFSS, finding a path is more than 80 times faster by using the nSCDF and we get roughly $3$ times more clearance. Next, we inspect our proposed method steps. Projecting the path onto the medial axis is the most time consuming part. However, it pays off, clearance is increased by more than $40~\%$, but at the same time, we see an increase in path length. Solving the WSP problem results in improved path length while more or less maintaining the clearance at a minor computational cost. The final smoothing step makes minor reductions to path length and clearance, while giving the necessary smoothness that is needed for time-scaling. Overall, our complete pipeline takes on average roughly $1\:s$, i.e., we are on average twice as fast compared to the EFSS method. We also end up with roughly $5$ times the clearance and $20\:\%$ lower path length. It should further be noted that the produced path from the EFSS method is in need of post-processing, as it is nonsmooth and, as seen from the path length, contains overly long motion. Adding post-processing would add even more computation time, which further demonstrates the efficiency of our method.
\subsection{Funnel Control}
\label{sec:exp:FC}
\subsubsection{Setup}
Our funnels are constructed from $10$ random world space instances, using the same setup from Section~\ref{sec:exp:FP}. Each funnel is time-scaled using~\cite{toppra}, limiting the velocity and acceleration to $0.5~rad/s$ and $0.5~rad/s^2$, respectively. To show the benefits of having a large funnel, we consider, in addition to using the nSCDF, two fixed sized funnels, with radius $0.10\:rad$ and $0.15\:rad$, respectively. We simulate the robot dynamics \eqref{eq:dyn_snd}. We use $\rho_{v_i}(t)=1$ for all $i \in \intRange_{1:\nrDOF}$, which satisfies the requirement $\rho_{v_i}(0) > |e_{v_i}(0)|$, and set $\funnelGain=\funnelVelocityGain=1$. Performance is measured based on the level of disturbance the funnel can withstand, considering both model and external disturbances. We also consider the motion performance, by quantifying the resulting jerk through $\int_0^{\timeFinal} \norm{\dddot{\q}(t)} dt$. We simulate two scenarios, which we now discuss. 
\\\\
In the first scenario we study the combined effect of model uncertainties and external disturbances. Model uncertainty is simulated by a biased estimate of the link masses, which affects all parts of the robot dynamics. We vary the level of uncertainty, from $5\%$ to $15\%$. The external disturbance is modeled as a sinusoidal, $\dist(t) = \alpha \cdot (10^{-1},1, 1, 10^{-1}, 10^{-1}, 10^{-3}) \cdot \sin(10 \: t)$, where the amplitude is varied using $\alpha \in \{10^{-3}, 10^{-2}, 10^{-1}\}$.
\\\\
In the second scenario, we consider a more practical scenario and simulate a digital control implementation, applying a piecewise-constant control input, updated every $1\:ms$. The uncertainty on the mass links are all set to $5\%$. We add a disturbance to our measured velocity, applying a sinusoid according to $A \cdot \sin(20 \pi \cdot t)$. The amplitude, $A$, is varied by considering $10$ values, logarithmically spaced between $10^{-4}$ and $10^{0}$. If the controller failed in any of the $10$ funnels, then the results were disregarded for that disturbance value, since an accurate statistics could not be formed. 
\subsubsection{Results}
The results from both scenarios are presented in Figure~\ref{fig:exp}. Starting with the first scenario, illustrated in the left part of the figure, we see a clear trend: increasing the funnel width lowers the jerk. In the second scenario, illustrated in the right part of Figure~\ref{fig:exp}, we observe the same trend. For lower levels of disturbance, we see a distinct difference in jerk performance. The smallest fixed sized funnel is performing the worst and using the nSCDF results in the best performance, since lower jerk is achieved, while it can continue to operate under higher levels of disturbances.
\section{Conclusions}
Funnel control is an attractive approach for control of manipulators with uncertain dynamics and disturbances. However, current state-of-the-art for funnel planning is computationally expensive, typically resulting in narrow funnels to reduce computation, which can have a negative effect on the motion performance and robustness margins. We approached the problem by first introducing a novel nSCDF for 6 DOF manipulators, which can process parameterizable obstacle shapes. As a next step, the nSCDF was integrated into both the funnel planning and control. Our funnel planner constructed paths, where we balanced both path length and clearance. It consisted of three steps: approximate medial-axis projection for higher clearance, weighted shortest path search to balance path length and clearance, and, finally, path smoothening. Subsequently, we presented a novel funnel control formulation, which utilized the geometrical shape of the nSCDF. In the numerical experiments, we observed that the funnel planning was greatly simplified by using the nSCDF, producing funnels twice as fast compared to our benchmark, while gaining $5$ times more clearance and $20 \%$ lower path length. In the funnel control experiments, we observed that larger funnels provide two key benefits: lower jerk and robustness to higher amplitudes of the disturbance.
\bibliographystyle{ieeetr}
\bibliography{main}
\appendix
\subsection{Proof of Theorem \ref{thm:ppc}}
\label{app:proof}
\begin{proof} 
The proof follows closely the proof of \cite{kdf,verginis2019robust}. 
We set $\Omega = \{(\bm{x},t) \in \mathbb{R}^{2n} \times [0,t_f): [\xi,\bm{\xi}_v^\top]^\top \in (-1,1)^{n+1} \}$ which is open and satisfies $ (\bm{x}(0),0)\in \Omega$. It can be verified that the right-hand-side of the closed-loop system obtained from \eqref{eq:dyn_snd} and \eqref{eq:control law} is locally Lipschitz in $\bm{x}$ and continuous in $t$ over $\Omega$. Hence, according to  \cite[Theorem 2.1.1]{bressan2007introduction}, we conclude the existence of a maximal solution $\bm{x}(t)$ that satisfies $(\bm{x}(t),t) \in \Omega$ for all $t\in[0,t_{\max})$, with $t_{\max} > 0$. 
What remains to be proven is $t_{\max} = \infty > t_f$. To that end, consider first the time derivative of $\varepsilon$ as well as $\qdot = \bm{v}_r(\xi,t) + \bm{e}_v$:
\begin{align*}
    \dot{\varepsilon} =& 2 J_\varepsilon(\xi)\rho(t)^{-2}\left( \funcErr^\top (\qdot - \dot{\funcTraj}(t))  -  \rho(t) \dot{\rho}(t) \xi \  \right)  \\
    =&  - 2 k_1 J_\varepsilon(\xi)^2 \rho(t)^{-2} \|\funcErr\|^2 + 2 J_\varepsilon(\xi)\rho(t)^{-2} \bm{F}_1 
\end{align*}
where $\bm{F}_1 =  \funcErr^\top(\bm{e}_v   - \dot{\funcTraj}(t)) -  \rho(t) \dot{\rho}(t) \xi$. By employing the positivity of $\rho(t)$, the boundedness of $\rho(t)$, $\dot{\rho}(t)$, and \eqref{eq:ksi}, and the fact that $| e_{v_i}(t)| \leq \rho_{v_i}(t)$, $\xi \leq \sqrt{\xi}$ (since $\xi < 1$ in $\Omega$) for all $t\in[0,t_{\max})$, we conclude that 
\small
\begin{align*}
    \dot{\varepsilon} \leq - \mu_1 k_1 J_\varepsilon(\xi)^2 \xi + \mu_2 J_\varepsilon(\xi) \xi 
    \leq   - \mu_1 k_1 J_\varepsilon(\xi)^2 \xi + \mu_2 J_\varepsilon(\xi) \sqrt{\xi}
\end{align*}
\normalsize
for all $t\in[0,t_{\max})$ and positive constants $\mu_1$ and $\mu_2$. Therefore, $\dot{\varepsilon} < 0$ when $J_\varepsilon(\xi) \sqrt{\xi} > \frac{\mu_2}{k_1\mu_1}$. According to Proposition 1 of \cite{verginis2019robust} and in view of $\varepsilon \geq 0$, it holds that $J_\varepsilon(\xi) \sqrt{\xi} = \sqrt{J_\varepsilon(\xi)^2 \xi} = \sqrt{\exp(\varepsilon)(\exp(\varepsilon) - 1)} \geq \varepsilon$, which implies that $\dot{\varepsilon} < 0$ when $\varepsilon > \frac{\mu_2}{k_1\mu_1}$. 
Therefore, according to \cite{khalil_nonlinear_systems}, we conclude the boundedness of $\varepsilon$ as $\varepsilon(\xi(t)) \leq \bar{\varepsilon} = \max\left\{ \varepsilon(0), \frac{\mu_2}{k_1\mu_1} \right\}$ for all $t\in[0,t_{\max})$, which, in view of \eqref{eq:epsilon}, implies that $\xi(t) \leq \bar{\xi} = {\exp(\bar{\epsilon}) - 1}{\exp(\bar{\epsilon})} < 1$ and $J_\varepsilon(\xi) \leq J_\varepsilon(\bar{\xi}) = \frac{1}{1-\bar{\xi}}$, for all $t\in[0,t_{\max})$. 
Consequently, we conclude the boundedness of the following signals: $\|\funcErr(t)\| \leq \bar{e} = \sqrt{ \bar{\xi}(t) \max_{t\in[0,t_f]}\{\rho(t)}\}$, $\|\q(t)\| \leq \bar{e} + \max_{t\in[0,t_f]}\{\|\funcTraj(t)\|\}$,
$\|\bm{v}_r(\xi(t),t)\| \leq \bar{v}_v = k_1 J_\varepsilon(\bar{\xi}) \bar{e}$, and $\|\qdot(t)\| \leq  \bar{v}_r + \max_{t\in[0,t_f]}\{\|\rho_{v}(t)\|\}\sqrt{n}$ for all $t\in[0,t_{\max})$. Finally, by differentiation $\bm{v}_r(\xi,\funcErr)$ and using the boundedness of the aforementioned signals, we also conclude the boundedness of $\dot{\bm{v}}_r(\xi(t),\funcErr(t))$ for all $t\in[0,t_{\max})$.
Consider next the function $V = \frac{1}{2}\|\bm{\varepsilon}_v\|^2$, whose time derivative is 
\begin{align*}
    \dot{V} =& \bm{\varepsilon}_v^\top \bm{J}_v(\bm{\varepsilon}_v) \bm{\rho}_v(t)^{-1} \big( \bm{M}(\q)^{-1} (\bm{u} - \bm{C}(\q,\qdot)\qdot - \bm{g}(\q) \\
    &- \bm{w}(\q,\qdot,t)) - \dot{\bm{v}}_r(\xi,\funcErr) - \dot{\bm{\rho}}_v(t)\bm{\xi}_v) \big)
\end{align*}

By invoking the positivity of $\bm{\rho}_v(t)$, the boundedness of $\dot{\bm{\rho}}_v(t)$, $\q(t)$, $\qdot(t)$, $\dot{\bm{v}}_r(\xi(t),\funcErr(t))$, $\bm{\xi}_v(t)$, the continuity of $\bm{C}(\cdot)$, $\bm{g}(\cdot)$, $\bm{w}(\cdot,t)$, $\hat{\bm{C}}(\cdot)$, $\hat{\bm{g}}(\cdot)$,  the boundedness of $\bm{w}(\q,\qdot,\cdot)$, and the positive definiteness of $\bm{M}(\q)$, we conclude that 

\begin{align*}
    \dot{V} \leq -k_2 \mu_3 \| \bm{J}_v(\bm{\varepsilon}_v) \bm{\varepsilon}_v \|^2 +  \mu_4 \| \bm{J}_v(\bm{\varepsilon}_v) \bm{\varepsilon}_v \| 
\end{align*}
for all $t\in[0,t_{\max})$ and positive constants $\mu_3$, $\mu_4$. Hence, we conclude that $\dot{V} < 0$ when $\| \bm{J}_v(\bm{\varepsilon}_v) \bm{\varepsilon}_v \|  > \frac{\mu_4}{k_2 \mu_3}$, which, by exploiting the fact that $\bm{J}_{v_i}(\bm{\xi}_{v_i}(t)) \geq 2$, for all $t\in[0,t_{\max})$ and $i\in\{1,\dots,n\}$, implies the boundedness of $\bm{\varepsilon}_v$ as $\|\bm{\varepsilon}_v(\bm{\xi}_v(t))\| \leq \bar{\varepsilon}_v = \max\{ \| \bm{\varepsilon}_v(\bm{\xi}_v(0)) \|, \frac{\mu_4}{2k_2 \mu_3} \}$ and hence 
$|\xi_{v_i}(t)| \leq \frac{\exp(\bar{\varepsilon}_v)-1}{\exp(\bar{\varepsilon}_v)+1} < 1$, 
for all $t\in[0,t_{\max})$ and $i\in\{1,\dots,n\}$. Consequently, the control law \eqref{eq:control law} and the closed-loop signals remain bounded for all $t\in[0,t_{\max})$. 

Finally, the bounds on $\xi(t)$ and $|\xi_{v_i}(t)|$ imply the containment of $(\bm{x}(t),t)$ in a compact subset of $\Omega$. Hence, by invoking a forward-completeness result (e.g., \cite[Theorem 2.1.4]{bressan2007introduction}), we conclude that $t_{\max} = \infty$, concluding the proof.
\end{proof}

\end{document}